\def\year{2022}\relax
\documentclass[letterpaper]{article} 
\usepackage{aaai22}  
\usepackage{times}  
\usepackage{helvet}  
\usepackage{courier}  
\usepackage[hyphens]{url}  
\usepackage{graphicx} 
\usepackage{natbib}  
\usepackage{caption} 
\DeclareCaptionStyle{ruled}{labelfont=normalfont,labelsep=colon,strut=off} 
\usepackage{algorithm}
\usepackage{algorithmic}
\usepackage{siunitx}

\title{Differential Assessment of Black-Box AI Agents
}
\author{
    Rashmeet Kaur Nayyar\equalcontrib,
    Pulkit Verma\equalcontrib,
    {\normalfont and}
    Siddharth Srivastava
}
\affiliations{
    Autonomous Agents and Intelligent Robots Lab,\\
    School of Computing and Augmented Intelligence, Arizona State University, AZ, USA\\
    \{rmnayyar, verma.pulkit, siddharths\}@asu.edu
}

\usepackage{amssymb,amsmath,amsthm}
\usepackage{multirow}
\usepackage{stmaryrd }
\usepackage{pict2e}
\usepackage[svgnames,table]{xcolor}
\usepackage{amsthm} 
\usepackage{thmtools}

\usepackage{pifont}
\newcommand{\cmark}{\ding{51}}%
\newcommand{\xmark}{\ding{55}}

\usepackage{booktabs}
\usepackage{fancyvrb}

\theoremstyle{definition} \newtheorem{theorem}{Theorem}
\newtheorem{lemma}[theorem]{Lemma}

\newtheorem{definition}{Definition}

\newcommand{\tuple}[1]{\langle#1 \rangle}

\renewcommand{\a}{\mathcal{A}}
\newcommand{\ag}{\mathcal{A}}
\newcommand{\m}{\mathcal{M}}

\newcommand{\mi}{M_{\emph{init}}^\ag}
\newcommand{\md}{M_{\emph{drift}}^\ag}
\newcommand{\mdset}{\mathcal{M}^{\ag}_{\emph{drift}}}

\newcommand{\ai}{\mathcal{A}_{\emph{init}}}
\newcommand{\ad}{\mathcal{A}_{\emph{drift}}}

\newcommand{\unkwn}{\raisebox{.5pt}{\textcircled{\raisebox{-.9pt} {?}}}}
\newcommand{\abs}{\emptyset}

\usepackage{tikz}
\usetikzlibrary{shapes.misc}
\newcommand{\mytimes}{ \tikz[baseline=-.55ex] \node [inner sep=0pt,cross out,draw,line width=1pt,minimum size=1ex] (a) {};}

\definecolor{aia-x-color}{HTML}{990000}

\begin{document}

\maketitle

\begin{abstract}

Much of the research on learning symbolic models of AI agents
focuses on agents with stationary models. This assumption fails
to hold in settings where the agent’s capabilities may change 
as a result of learning, adaptation, or other post-deployment 
modifications. Efficient assessment of agents in such settings 
is critical for learning the true capabilities of an AI system 
and for ensuring its safe usage. In this work, we propose a 
novel approach to $\emph{differentially}$ assess black-box AI 
agents that have drifted from their previously known models. As a 
starting point, we consider the fully observable and deterministic
setting. We leverage sparse observations of the drifted agent’s 
current behavior and knowledge of its initial model to generate
an active querying policy that selectively queries the agent 
and computes an updated model of its functionality. Empirical 
evaluation shows that our approach is much more efficient than 
re-learning the agent model from scratch. We also show that the
cost of differential assessment using our method is 
proportional to the amount of drift in the agent’s 
functionality.
\end{abstract}

\section{Introduction} 
\label{sec:introduction}

With increasingly greater autonomy in AI systems in recent
years, a major problem still persists and has largely been 
overlooked: how do we accurately predict the behavior of a 
black-box AI agent that is evolving and adapting to changes in 
the environment it is operating in? And how do we ensure its 
reliable and safe usage? Numerous factors could cause 
unpredictable changes in agent behaviors: sensors and actuators may
fail due to physical damage, the agent may adapt to a dynamic 
environment, users may change deployment and use-case scenarios, etc. 
Most prior work on the topic presumes that the functionalities and 
the capabilities of AI agents are static, while some works start 
with a \emph{tabula-rasa} and learn the entire model from 
scratch. However, in many real-world scenarios, the agent model
is transient and only parts of its functionality change 
at a time. 

\citet{Bryce2016} address a related problem 
where the system learns the updated mental model of a user 
using particle filtering given prior knowledge about the user's
mental model. However, they assume that the entity being modeled can
tell the learning system about flaws in the learned model if needed. 
This assumption does not hold in settings where the entity being modeled
is a black-box AI system:  most such systems are either implemented using
inscrutable representations or otherwise lack the ability to automatically
generate a model of their functionality (what they can do and when) in
terms the user can understand. The problem of efficiently assessing, in
human-interpretable terms, the functionality of such a non-stationary AI
system has received little research attention.

\begin{figure}[t]
\centering
\includegraphics[width=\linewidth]{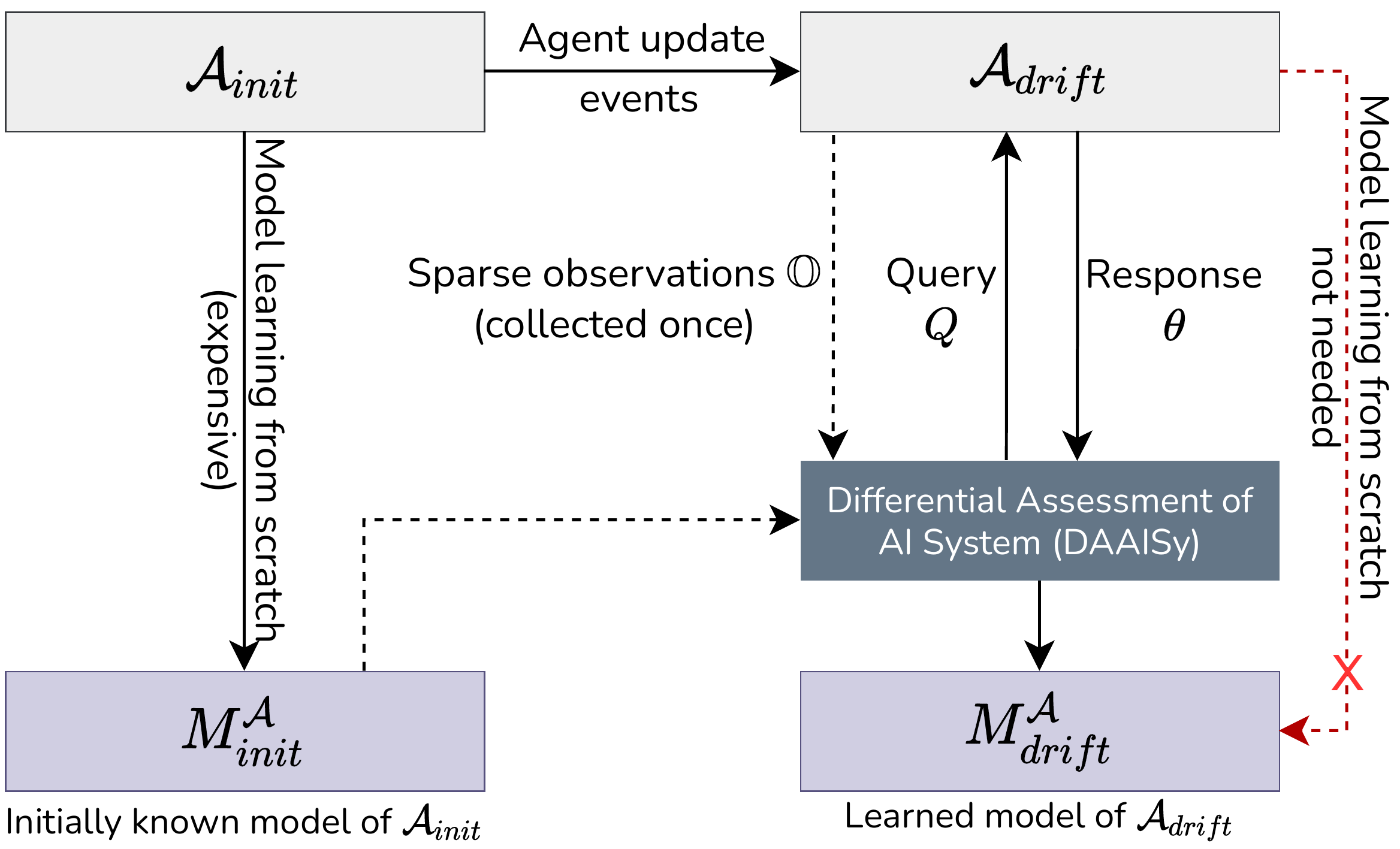}
\caption{The Differential Assessment
of AI System (DAAISy) takes as input the initially known model of the agent 
prior to model drift, available observations of the updated agent's behavior, and performs a selective dialog with the black-box AI
agent to output its updated model through efficient model learning. }
\label{fig:blockdiagram}
\end{figure}

The primary contribution of this paper is an algorithm for \emph{differential assessment}  of black-box AI systems (Fig.~\ref{fig:blockdiagram}). This algorithm utilizes an initially known interpretable model of the agent as it was in the past, and a small set of observations of agent execution. It uses these observations to develop an incremental querying strategy that avoids the full cost of assessment from scratch and outputs a revised model of the agent's new functionality. 
One of the challenges in learning agent models from observational data is that reductions in agent functionality often do not correspond to specific ``evidence'' in behavioral observations, as the agent may not visit states where certain useful actions are no longer applicable. Our analysis shows that if the agent can be placed in an ``optimal'' planning mode, differential assessment can indeed be used to query the agent and recover information about reduction in functionality. 
This ``optimal'' planning mode is not necessarily needed for learning about increase in functionality.
Empirical evaluations on a range of problems clearly demonstrate that our method is much 
more efficient than re-learning the agent's model from scratch. 
They also exhibit the desirable property that the computational cost of differential assessment is
proportional to the amount of drift in the agent's functionality. 

\subsubsection{Running Example} 
Consider a battery-powered rover with limited storage capacity
that collects soil samples and takes pictures. Assume that its planning model 
is similar to IPC domain Rovers~\cite{long_2003_ipc}.
It has an action that collects a rock sample at a waypoint and stores it in a storage iff 
it has at least half of the battery capacity remaining.
Suppose there was an update to the rover's system and as a result of this update, the rover
can now collect the rock sample only when its battery is full,
as opposed to at least half-charged battery that it needed before. 
Mission planners familiar with the earlier system and unaware about the exact updates in the functionality of the rover would struggle to collect sufficient samples. This could jeopardise multiple
missions if it is not detected in time.

This example illustrates how our system could be of value by differentially detecting such a drift in the functionality of a black-box AI system and deriving its true functionality.

The rest of this paper is organized as follows: The next section presents background terminology. This is followed by a formalization of the differential model assessment problem in Section 3. Section 4 presents our approach for differential assessment by first identifying aspects of the agent's functionality that may be affected (Section 4.1) followed by the process for selectively querying the agent using a primitive set of queries. We present empirical evaluation of the efficiency of our approach on randomly generated benchmark planning domains in Section 5. Finally, we discuss relevant related work in  Section 6 and conclude in Section 7.

\section{Preliminaries}
\label{sec:background}

We consider models that express an agent's functionalities in the form of STRIPS-like planning models~\cite{Fikes1971,McDermott_1998_PDDL,fox03_pddl} as defined below.

\begin{definition}
A planning domain model is a tuple $M = \langle P, A \rangle$, where $P = \{p_{1}^{r_1},\dots, p_{n}^{r_n} \}$ is a finite set of predicates with arities $r_i$, $i \in [1,n]$; and $A = \{a_1,\dots, a_k \}$ is a finite set of parameterized relational actions. Each action $a_i \in A$ is represented as a tuple $\langle header(a_{i}), \emph{pre}(a_{i}), \emph{eff}(a_i) \rangle $, where $header(a_{i})$ represents the action header consisting of the name and parameters for the action $a_i$, $pre(a_i)$ represents the conjunction of positive or negative literals that must be true in a state where the action $a_i$ is applicable, and $\emph{eff}(a_i)$ is the conjunction of positive or negative literals that become true as a result of execution of the action $a_i$. 
\end{definition}

In the rest of the paper, we use the term ``model'' to refer to planning domain models and use closed-world assumption as used in the Planning Domain Definition Language (PDDL)~\cite{McDermott_1998_PDDL}.
Given a model $M$ and a set of objects $O$, let $S_{M,O}$ be the space of all states defined as maximally consistent sets of  literals over the predicate vocabulary of $M$ with $O$ as the set of objects. We omit the subscript when it is clear from context.
An action $a \in A$ is applicable in a state $s \in S$ if $s \models
\emph{pre}(a)$. The result of executing $a$ is a state $a(s)=s' \in S$ such that $s'\models \emph{eff}(a)$, and all atoms not in $\emph{eff}(a)$  have literal forms as in $s$.

A literal corresponding to a predicate $p \in P$ can appear in
$\emph{pre}(a)$ or $\emph{eff}(a)$ of an action $a \in A$ if and only if 
it can be instantiated using a subset of parameters of $a$. 
E.g., consider an action \emph{navigate (?rover ?src ?dest)} and a predicate
\emph{(can\_traverse ?rover ?x ?y)} in the Rovers domain discussed earlier. Suppose a literal
corresponding to predicate \emph{(can\_traverse ?rover ?x ?y)} can appear in the precondition 
and/or the effect of \emph{navigate (?rover ?src ?dest)} action. 
Assuming we know \emph{?x} and \emph{?y} in \emph{can\_traverse}, and  \emph{?src} and \emph{?dest} in \emph{navigate} are of the same type 
\emph{waypoint}, the possible lifted instantiations of 
predicate \emph{can\_traverse} compatible with action \emph{navigate} 
are \emph{(can\_traverse ?rover ?src ?dest)}, 
\emph{(can\_traverse ?rover ?dest ?src)}, 
\emph{(can\_traverse ?rover ?src ?src)}, and 
\emph{(can\_traverse ?rover ?dest ?dest)}.
The number of parameters in a predicate $p \in P$ that is relevant to an 
action $a \in A$, i.e., instantiated using a subset of parameters of the 
action $a$, is bounded by the maximum arity 
of the action $a$. We formalize this notion of lifted instantiations of 
a predicate with an action as follows:

\begin{definition}
Given a finite set of predicates $P = \{p_{1}^{r_1},\dots, p_{n}^{r_n} \}$ with 
arities $r_i$, $i \in [1,n]$; and
a finite set of parameterized relational actions $A = \{a_1^{\psi_1},\dots, a_k^{\psi_k} \}$
with arities $\psi_j$ and parameters 
$\emph{par}(a_{j}^{\psi_j}) = \langle \alpha_1,\dots,\alpha_{\psi_j}\rangle$, $j \in [1,k]$,
the set of \emph{lifted instantiations of predicates} $P^*$ is defined as the collection 
$ \{p_i(\sigma(x_1),\dots,\sigma(x_{r_i})) \,| p_i \in P, a \in A, \sigma:\{x_1,\dots,x_{r_i} \}\rightarrow \emph{par}(a)\}$.
\end{definition}

\subsection{Representing Models} 
We represent a model $M$ using the set of all possible \emph{pal-tuples} $\Gamma_M$ of 
the form $\gamma = \langle p, a, \ell \rangle$, where $a$ is a parameterized action header for an action in $A$, $p \in P^*$ is a possible lifted instantiation of a predicate in $P$,
and $\ell \in$ \{\emph{pre}, \emph{eff}\} denotes a location in $a$, precondition or effect, where $p$ can appear. A model $M$ is
thus a function $\mu_M : \Gamma_M \rightarrow \{+, - , \abs \}$ that maps each element in $\Gamma_M$ to a \emph{mode} in the set $\{+, - , \abs \}$. The assigned mode for a \emph{pal-tuple} $\gamma \in \Gamma_M$ denotes whether  $p$ is present as a positive literal ($+$), as a negative literal ($-$), or absent ($\emptyset$) in the precondition ($\ell = $\emph{pre}) or effect ($\ell = $ \emph{eff}) of the action header $a$. 

This formulation of models as \emph{pal-tuples}
allows us to view the modes 
for any predicate in an action's precondition and effect 
independently. However, 
at times it is useful to consider a model at a granularity of
relationship between a predicate and an action.
We address this by representing a model $M$ as a set of
\emph{pa-tuples} $\Lambda_M$ of the form $\langle p, a\rangle$ where $a$ is a parameterized action header for an action in $A$, and $p \in P^*$ is a possible lifted instantiation of a predicate in $P$. Each \emph{pa-tuple}
can take a value of the form $\langle m_{\emph{pre}}, m_{\emph{eff}} \rangle$,
where $m_{\emph{pre}}$ and $m_{\emph{eff}}$ represents the mode in which
$p$ appears in the precondition and effect of $a$, respectively. 
Since a predicate cannot appear as a positive (or negative) literal in
both the precondition and effect of an action, $\langle +,+ \rangle$
and $\langle -, - \rangle$ are not in the range of values that 
\emph{pa-tuples} can take. Henceforth, in the context of a \emph{pal-tuple} or a \emph{pa-tuple}, we refer to $a$ as an action instead of an action header.

\subsubsection{Measure of model difference} 
Given two models $M_1 = \langle 
P, A_1 \rangle$ and $M_2 = \langle P, A_2 \rangle$, defined over the same sets of predicates $P$ and action headers $A$, 
the difference between the two models $\Delta (M_1,M_2)$ is defined as the number of 
\emph{pal-tuples} that differ in their modes in $M_1$ and $M_2$, i.e., $\Delta (M_1,M_2) = |\{ \gamma \in P\times A \times \{ +, - ,\emptyset\} | \mu_{M_1}({\gamma}) \neq \mu_{M_2}({\gamma}) \} |$.

\subsection{Abstracting Models} Several authors have explored the use of abstraction in planning~\cite{Sacerdoti74,Giunchiglia92,Helmert07,Backstrom13,Srivastava16}. 
We define an abstract model as a model that does not have a mode assigned
for at least one of the \emph{pal-tuples}.
Let $\Gamma_M$ be the set of all
possible \emph{pal-tuples}, and $\unkwn$ be an additional possible value that a \emph{pal-tuple} can take. Assigning $\unkwn$ mode to a \emph{pal-tuple} denotes that its mode is unknown. An abstract model $M$ is
thus a function $\mu_M : \Gamma_M \rightarrow \{+, - , \abs, \unkwn \}$ that maps each element in $\Gamma_M$ to a \emph{mode} in the set $\{+, - , \abs, \unkwn \}$.
Let $\mathcal{U}$ be the set of all 
abstract and concrete models that can possibly be expressed by assigning modes in $\{+,-,\abs,\unkwn\}$ to each \emph{pal-tuple} $\gamma \in \Gamma_M$. We now formally define model abstraction as follows:

\begin{definition}
\label{def:abstraction}
Given models $M_1$ and $M_2$, $M_2$ is an abstraction of $M_1$ over the set of all possible \emph{pal-tuples} $\Gamma$ iff $\exists \Gamma_2 \subseteq \Gamma $ s.t. $\forall \gamma \in \Gamma_2$, $\mu_{M_2}(\gamma)=\unkwn $
and $\forall \gamma \in \Gamma \setminus \Gamma_2$, $\mu_{M_2}(\gamma) = \mu_{M_1}(\gamma)$.
\end{definition}

\subsection{Agent Observation Traces} We assume limited access to 
a set of observation traces $\mathbb{O}$,
collected from the agent, as defined below.
\begin{definition}
An \emph{observation trace} $o$ is a sequence of states and actions
of the form $\langle s_0, a_1, s_1, a_2, \dots,
s_{n-1}, a_n, s_n \rangle$ such that $\forall i\in [1,n]  \;$
 $a_i(s_{i-1}) = s_i$.
\end{definition}

These observation traces can be split into multiple action triplets as defined below.
\begin{definition}
Given an observation trace $o = \langle s_0, a_1, s_1, a_2, \dots,
s_{n-1}, a_n, s_n \rangle$, an \emph{action triplet} is a 3-tuple sub-sequence of $o$ of the form
$\langle s_{i-1}, a_i, s_{i} \rangle$, where $i\in [1,n]$ and applying 
an action $a_i$ in state $s_{i-1}$
results in state $s_{i}$, i.e., $a_i(s_{i-1}) = s_{i}$. 
The states $s_{i-1}$ and $s_{i}$ are called pre- and post-states of action $a_i$, respectively.
\end{definition}

An action triplet $\langle s_{i-1}, a_{i}, s_{i} \rangle$ is said to be \emph{optimal} if there does not exist an action sequence (of length $\geq 1$) that takes the agent
from state $s_{i-1}$ to $s_{i}$
with total action cost less than that of action $a_{i}$, where each action $a_{i}$ has unit cost.

\subsection{Queries}
\label{sec:query}

We use queries to actively gain information about the functionality of an agent to learn its updated model. We assume that the agent can respond to a query using a simulator. The availability of such agents with simulators is a common assumption as
most AI systems already use simulators for design, testing, and verification.

We use a notion of queries
similar to \citet{verma2021asking}, to perform a dialog with an autonomous agent. These queries use an agent to determine what happens if it executes a sequence of actions in a given initial state. 
E.g., in the rovers domain,
the rover could be asked: what happens when the action \emph{sample\_rock(rover1 storage1 waypoint1)} is executed in
an initial state $\{$\emph{(equipped\_rock\_analysis rover1), (battery\_half rover1), (at rover1 waypoint1)}$\}$?

Formally, a \emph{query} is a function that maps an agent to a response, which we define as:

\begin{definition}
\label{def:uquery}
Given a set of predicates $P$, a set of actions $A$, and a set of objects $O$, a \emph{query}
$Q \langle s, \pi \rangle: \mathcal{A} \rightarrow 
\mathbb{N} \times S$  is parameterized by a start state $s_I \in S$
and a plan $\pi = \langle a_1,\dots,a_N \rangle$, where 
$S$ is the state space over $P$ and $O$, and 
$\{ a_1,\dots,a_N \}$ is a subset of action space over $A$ and $O$.
It maps agents to responses $\theta= \langle n_F, s_F \rangle$ such that
$n_F$ is the length of the longest prefix of  $\pi$ that $\mathcal{A}$ can
successfully execute and $s_F \in S$ is the result of that execution.
\end{definition}

Responses to such queries can be used to gain useful information
about the model drift. 
E.g., consider an agent with an internal model $\md$ as shown in Tab.~\ref{tab:turn_to}.
If a query is posed asking what
happens when the action \emph{sample\_rock(rover1 storage1 waypoint1)} is executed in
an initial state $\{$\emph{(equipped\_rock\_analysis rover1), (battery\_half rover1), (at rover1 waypoint1)}$\}$, the agent would respond $\langle 0, \{$\emph{(equipped\_rock\_analysis rover1), (battery\_half rover1), (at rover1 waypoint1)}$\} \rangle$, representing that it was not able to execute the plan, and the resulting state was $\{$\emph{(equipped\_rock\_analysis rover1), (battery\_half rover1), (at rover1 waypoint1)}$\}$ (same as the initial state in this case). Note that this response is inconsistent with the model $\mi$, and it can help in identifying that the precondition of action \emph{sample\_rock(?r ?s ?w)} has changed.

\section{Formal Framework}
\label{sec:formulation}

Our objective is to address the problem of differential assessment of black-box AI agents whose functionality may have changed from the last known model. Without loss of generality, we consider situations where the set of action headers is same because the problem of differential assessment with changing action headers can be reduced to that with uniform action headers. This is because if the set of actions has increased, new actions can be added with empty preconditions and effects to $\mi$, and if it has decreased, $\mi$ can be reduced similarly. We assume that the predicate vocabulary used in the two models is the same; extension to situations where the vocabulary changes can be used to model open-world scenarios. However, that extension is beyond the scope of this paper.

\begin{table}[t]
    \footnotesize
    \centering
    \begin{tabular}{p{0.12\columnwidth}@{}p{0.435\columnwidth}@{}p{0.025\columnwidth}@{} >{\columncolor{white}[\tabcolsep][0pt]}p{0.38\columnwidth}@{}}
    \toprule
    Model & {$\,$Precondition} & & Effect\\
    \midrule
        $M^\ag_{\emph{init}}$ & \emph{(equipped\_rock\_analysis\!\,?r)} $\,\,\,\,\,\,\,\,\,\,$
        \emph{(battery\_half\! ?r)} $\,\,\,\,\,\,\,\,\,\,\,\,\,\,\,\,\,\,\,\,\,\,\,$
        \emph{(at ?r ?w)} & 
        $\rightarrow$ & 
        \emph{(rock\_sample\_taken\!\,?r)} $\,\,\,$
        \emph{(store\_full ?r ?s)} $\,\,\,\,\,\,\,\,\,\,\,\,\,\,\,\,\,\,\,$
        $\neg$\emph{(battery\_half\! \!\!?r)} $\,\,\,\,\,\,\,\,\,$
        \emph{(battery\_reserve ?r)}
        \\
        \midrule
        $M^\ag_{\emph{drift}}$ & \emph{(equipped\_rock\_analysis\!\,?r)} $\,\,\,\,\,\,\,\,\,\,\,\,\,\,\,\,\,\,\,\,\,\,\,\,\,\,$
        \emph{(battery\_full\! ?r)} $\,\,\,\,\,\,\,\,\,\,\,\,\,\,\,\,\,\,\,\,\,\,\,\,$
        \emph{(at ?r ?w)} & 
        $\rightarrow$ & 
        \emph{(rock\_sample\_taken\! ?r)} $\,\,\,$
        \emph{(store\_full ?r ?s)} $\,\,\,\,\,\,\,\,\,\,\,\,\,\,\,\,\,\,\,$
        $\neg$\emph{(battery\_full ?r)} $\,\,\,\,\,\,\,\,\,$
        \emph{(battery\_half ?r)}\\
        \bottomrule
    \end{tabular}
    \caption{\normalfont \emph{sample\_rock (?r ?s ?w)} action of the agent $\ag$ in $M^\ag_{\emph{init}}$ and a possible drifted model $M^\ag_{\emph{drift}}$.}
    \label{tab:turn_to}
\end{table}

Suppose an agent $\ag$'s functionality was known as a
model $\mi = \langle P, \ai \rangle$, and
we wish to assess its current functionality as the model
$\md =  \langle P, \ad \rangle$. 
The drift in the functionality of the agent can be measured
by changes in
the preconditions and/or effects of all the actions in $\ai$.
The extent of the drift between $\mi$ and $\md$ is represented as the model difference 
$\Delta(\mi,\md)$.

We formally define the problem of differential assessment of an AI agent below.

\begin{definition}
Given an agent $\ag$ with a functionality model $\mi$, and a set of
observations $\mathbb{O}$ collected using its 
current version of $\ad$ with unknown functionality $\md$, the
\emph{differential model assessment} problem
$\langle \mi, \md, \mathbb{O}, \ag \rangle$ is defined as the problem of
inferring $\ag$
in form of $\md$ using the inputs
$\mi$, $\mathbb{O}$, and $\ag$. 

\end{definition}
We wish to develop solutions to the problem of differential assessment of AI
agents that are more efficient than re-assessment from scratch.

\subsection{Correctness of Assessed Model}
We now discuss the properties that a model, which is a solution to the differential model assessment problem, should satisfy. A critical property of such models is that they should be consistent with the observation traces.
We formally define consistency of a model w.r.t. an observation trace as follows:

\begin{definition}
Let $o$ be an observation trace $\langle s_0, a_1, s_1, a_2, \dots,
s_{n-1}, a_n, s_n \rangle$.
A model $M = \langle P, A \rangle$ is \emph{consistent with the observation trace} $o$ iff
$\forall i \in \{1,..,n\}$ $\exists a \in A$ and $a_i$ is a grounding of action $a$ $\; s.t.\quad s_{i-1} \models \emph{pre}(a_i) \; \land \; \forall\, l \in \emph{eff}(a_i) \; s_i \models l$.
\end{definition}

In addition to being consistent with observation traces, a model should also be consistent with the queries that are asked and the responses that are received while actively inferring the model of the agent's new functionality.
We formally define consistency of a model with respect to a query and a response as:

\begin{definition}
Let $M = \langle P, A \rangle$ be a model; $O$ be a set of objects;  $Q =\langle s_I, \pi= \langle a_1,\dots a_n \rangle \rangle$ be a query defined using $P, A,$ and $O$, and let 
$\theta =\langle n_F, s_F\rangle$, ($n_F\le n$)  be a response to $Q$. $M$ is \emph{consistent with the query-response} $\langle Q, \theta \rangle$ iff there exists an observation trace $\langle s_I, a_1, s_1, \ldots, a_{n_F}, s_{n_F} \rangle$ that $M$ is consistent with and $s_{n_F}\not\models \emph{pre}(a_{n_F+1})$ where $\emph{pre}(a_{n_F+1})$ is the precondition of $a_{n_F+1}$ in $M$. 
\end{definition}

We now discuss our methodology for solving the problem of differential assessment of AI systems.

\section{Differential Assessment of AI Systems} 
\label{sec:approach}

\textbf{D}ifferential \textbf{A}ssessment of \textbf{AI} \textbf{Sy}stems (Alg.~\ref{alg:alg3}) -{}- DAAISy -{}-
takes as input an agent $\ag$ whose functionality has drifted, the model
$\mi = \langle P, A \rangle$ representing the previously known functionality of $\ag$, a set of arbitrary observation traces $\mathbb{O}$, and a set of random states $\mathcal{S} \subseteq S$. Alg.~\ref{alg:alg3} returns a set of updated models $\mdset$, where each model $\md \in \mdset$ represents
$\ag$'s updated functionality
and is consistent
with all observation traces $o \in \mathbb{O}$.

A major contribution of this work is to introduce an approach to
make inferences about not just the expanded functionality of an agent but also its reduced functionality using a limited set of observation traces.
Situations where the scope of applicability of an action reduces, i.e., the agent can no longer use an action $a$ to reach state $s'$ from state $s$ while it could before (e.g., due to addition of a precondition literal), are particularly difficult to identify because observing its behavior does not readily reveal what it cannot do in a given state. Most observation based action-model learners, even when given access
to an incomplete model to start with, fail to make inferences about reduced functionality. DAAISy uses two principles to identify such a functionality reduction. First, it uses active querying so that the agent can be made to reveal failure of reachability, and second, we show that if the agent can be placed in optimal planning mode, plan length differences can be used to infer a reduction in functionality.

DAAISy performs two major functions;
it first identifies a salient set of \emph{pal-tuples} whose modes were likely affected (line 1 of Alg.~\ref{alg:alg3}), and then
infers the mode of such affected \emph{pal-tuples} accurately through focused dialog with the agent (line 2 onwards of Alg.~\ref{alg:alg3}). 
In Sec.~\ref{sec:affected_pals}, we present our method for identifying a salient set of potentially affected \emph{pal-tuples} that contribute towards expansion in the functionality of the agent through inference from available arbitrary observations. We then discuss the problem of identification of \emph{pal-tuples} that contribute towards reduction in the functionality of the agent and argue that it cannot be performed using successful executions in observations of satisficing behavior.
We show that \emph{pal-tuples} corresponding to reduced functionality can be identified if observations of optimal behavior of the agent are available (Sec.~\ref{sec:affected_pals}). Finally, we present how we infer the nature of changes in all affected \emph{pal-tuples} through a query-based interaction with the agent (Sec.~\ref{sec:dmaa}) by building upon the Agent Interrogation Algorithm (AIA)~\cite{verma2021asking}. Identifying affected \emph{pal-tuples} helps reduce the computational cost of querying as opposed to the exhaustive querying strategy used by AIA.
We now discuss the two major functions of Alg.\,\ref{alg:alg3} in detail.

\begin{algorithm}[t]
	 \caption{Differential Assessment of AI Systems}
		\textbf{Input:}$M^{\a}_{\emph{init}}$, $\mathbb{O}$, $\a$, $\mathcal{S}$ \\
		\textbf{Output:} $\mathcal{M}^{\ag}_{\emph{drift}}$
		\begin{algorithmic}[1]
		\STATE $\Gamma_\delta \gets \emph{identify\_affected\_pals()}$ \;
		\STATE $M_{\emph{abs}} \!\gets\!$ set \emph{pal-tuples} in $\mi$ corresponding to $\Gamma_\delta$ to $\unkwn$\;
		\STATE $\mathcal{M}^\ag_{\emph{drift}}$  $ \gets \{M_{\emph{abs}}\}$\;
	
		\FOR{each $\gamma$ in $\Gamma_\delta$}
			\FOR{each $M_\emph{abs}$ in $\mathcal{M}^{\ag}_\emph{drift}$}
				\STATE $\mathcal{M}_\emph{abs} \gets M_\emph{abs} \times \{\gamma^+, \gamma^-, \gamma^\emptyset\}$\; 
				\STATE $\mathcal{M}_\emph{sieved} \gets \{\}$
						\IF{action corresponding to $\gamma$: $\gamma_a$ in $\mathbb{O}$}
						\STATE $s_\emph{pre} \gets \emph{states\_where\_}\gamma_a\emph{\_applicable}(\mathbb{O},\gamma_a)$\;
						\STATE $Q \gets \langle s_\emph{pre} \setminus \{\gamma_p \cup \neg \gamma_p$\}, $\gamma_a$ $\rangle$\;
						\STATE $\theta \gets \emph{ask\_query}(\a, Q)$ \;
						\STATE $\mathcal{M}_\emph{sieved} \gets \emph{sieve\_models}(\mathcal{M}_\emph{abs}, Q, \theta)$\;
						\ELSE
						\FOR{each pair $\langle M_i, M_j \rangle$ in $\mathcal{M}_\emph{abs}$}
							\STATE $Q \gets \emph{generate\_query}(M_i, M_j, \gamma, S)$\;
							\STATE $\theta \gets \emph{ask\_query}(\ag, Q)$\;
							\STATE $\mathcal{M}_\emph{sieved} \gets \emph{sieve\_models}(\{ M_i, M_j \}, Q, \theta)$\;
						\ENDFOR
						\ENDIF
				\STATE $\mathcal{M}_\emph{abs} \gets \mathcal{M}_\emph{abs} \setminus$ $\mathcal{M}_\emph{sieved}$
			  \ENDFOR
			  \STATE $\mathcal{M}^\ag_\emph{drift}$ $ \gets \mathcal{M}_\emph{abs}$\;
		\ENDFOR
		 \end{algorithmic}
		\label{alg:alg3}
	\end{algorithm}

\subsection{Identifying Potentially Affected pal-tuples}
\label{sec:affected_pals}

We identify a reduced set of \emph{pal-tuples} whose modes were potentially affected during the model drift, denoted by $\Gamma_\delta$, using a small set of available observation traces $\mathbb{O}$.
We draw two kinds of inferences from these observation traces: inferences about expanded functionality,
and inferences about reduced functionality. We discuss our method for inferring $\Gamma_\delta$ for both types of changes in the functionality below.

\subsubsection{Expanded functionality} To infer expanded functionality of the agent,
we use the previously known model of the agent's functionality and identify its differences with the possible behaviors of the agent that are consistent with $\mathbb{O}$.
To identify the \emph{pal-tuples} that directly contribute to an expansion in the agent's functionality, we perform an analysis similar to ~\citet{Stern2017}, but instead
of bounding the predicates that can appear in each action's precondition and effect, we bound the range of possible values that each \emph{pa-tuple} in $\md$ can
take using Tab.~\ref{tab:t2}. For any \emph{pa-tuple}, a direct comparison between its value in $\mi$ and possible inferred values in $\md$ provides an indication of whether it was affected.

	
\begin{table}[t]
	\begin{tabular}{c@{}c@{}c@{}c@{}c} 
	\toprule
	$\langle m_{\emph{pre}}, m_{\emph{eff}} \rangle$ & $\,\,\emph{(pos,pos)}\,\,$ & $\emph{(pos,neg)}$ & $\,\,\emph{(neg,pos)}\,\,$ & $\,\,\emph{(neg,neg)}$ \\
	\midrule
	$\langle +,- \rangle$ & \xmark & \cmark & \xmark & \xmark \\ 
	$\langle +,\,\emptyset\, \rangle$ & \cmark & \xmark & \xmark & \xmark \\
	\midrule
	$\langle -,+ \rangle$ & \xmark & \xmark & \cmark & \xmark \\ 
	$\langle -,\,\emptyset\, \rangle$ & \xmark & \xmark & \xmark & \cmark \\
	\midrule
	$\langle \,\emptyset\,, + \rangle$ & \cmark & \xmark & \cmark & \xmark \\
	$\langle \,\emptyset\,, - \rangle$ & \xmark & \cmark & \xmark & \cmark \\
	$\langle \,\emptyset\,, \,\emptyset\, \rangle$ & \cmark & \xmark & \xmark & \cmark \\
	\bottomrule
	\end{tabular}
	\centering
	\caption{\label{tab:t2} Each row represents a possible value $\langle m_{\emph{pre}}, m_{\emph{eff}} \rangle$ for a \emph{pa-tuple} $\langle p, a \rangle$. Each column represents a possible tuple representing presence of predicate $p$ in the pre- and post-states of an action triplet $\langle s_i, a, s_{i+1}\rangle$ (discussed in Sec.\ref{sec:affected_pals}). The cells represent whether a value for \emph{pa-tuple} is consistent with an action triplet in observation traces.}
	\end{table}

To identify possible values for a \emph{pa-tuple} $\langle p,a\rangle$, we first collect a set of all the action-triplets from $\mathbb{O}$ that contain the action $a$. For a given predicate $p$ and state $s$, if $s \models p$ then the presence of predicate $p$ is represented as $\emph{pos}$, similarly, if $s \models \neg p$ then the presence of predicate $p$ is represented as $\emph{neg}$. Using this representation, a tuple of predicate presence $\in \{\emph{(pos,pos)}$, $\emph{(pos,neg)}$, $\emph{(neg,pos)}$, $\emph{(neg,neg)} \}$ is determined for the \emph{pa-tuple} $\langle p,a\rangle$ for each action triplet $\langle s,a,s'\rangle \in \mathbb{O}$ by analyzing the presence of predicate $p$ in the pre- and post-states of the action triplets. Possible values of the \emph{pa-tuple} that are consistent with $\mathbb{O}$ are directly inferred from the Tab.~\ref{tab:t2} using the inferred tuples of predicate presence.
E.g., for a \emph{pa-tuple}, the values $\langle +, - \rangle$ and $\langle \emptyset, - \rangle$ are consistent with $(pos, neg)$, whereas, only $\langle \emptyset, + \rangle$ is consistent with $(pos, pos)$ and $(neg, pos)$ tuples of predicate presence that are inferred from $\mathbb{O}$. 

Once all the possible values for each \emph{pa-tuple} in $\md$ are inferred, we identify \emph{pa-tuples} whose previously known value in $\mi$ is no longer possible due to inconsistency with $\mathbb{O}$. The \emph{pal-tuples} corresponding to such \emph{pa-tuples} are added to the set of potentially affected \emph{pal-tuples} $\Gamma_\delta$.
Our method also infers the correct modes of a subset of \emph{pal-tuples}. E.g., consider a predicate $p$ and two actions triplets in $\mathbb{O}$ of the form $\langle s_1,a,s_1'\rangle$ and $\langle s_2,a,s_2'\rangle$ that satisfy $s_1 \models p$ and $s_2 \models \neg p$. Such an observation clearly indicates that $p$ is not in the precondition of action $a$, i.e., mode for $\langle p,a\rangle$ in the precondition is $\emptyset$. Such inferences of modes are used to update the known functionality of the agent. We remove such \emph{pal-tuples}, whose modes are already inferred, from $\Gamma_\delta$.

A shortcoming of direct inference from successful executions in available observation traces is that it cannot learn any reduction in the functionality of the agent, as discussed in the beginning of Sec. \ref{sec:approach}. We now discuss our method to address this limitation and identify a larger set of potentially affected \emph{pal-tuples}.

\subsubsection{Reduced functionality} 
We conceptualize reduction in functionality as an increase in the optimal cost of going from one state to another. More precisely, reduction in functionality represents situations where there exist states $s_i, s_j$ such that the minimum cost of going from  $s_i$ to $s_j$ is higher in $M^{\ag}_\emph{drift}$ than in $\mi$. In this paper, this cost refers to the number of steps between the pair of states as we consider unit action costs. This notion encompasses situations with reductions in reachability as a special case.
In practice, a reduction in functionality may occur if the precondition of at least one action in $M^{\ag}_\emph{drift}$ has new \emph{pal-tuples}, or the effect of at least one of its actions has new \emph{pal-tuples} that conflict with other actions required for reaching certain states. 

Our notion of reduced functionality captures all the variants of reduction in functionality. However, for clarity, we illustrate an example that focuses on situations where precondition of an action has increased. Consider the case from Tab.~\ref{tab:turn_to} where $\ag$'s model gets updated from $\mi$ to $M^{\ag}_\emph{drift}$. The action $\emph{sample\_rock}$'s applicability in $M^{\ag}_\emph{drift}$ has reduced from that in $\mi$ as $\ag$ can no longer sample rocks in situations where the battery is half charged but needs a fully charged battery to be able to execute the action. In such scenarios, instead of relying on observation traces, our method identifies traces containing indications of actions that were affected either in their precondition or effect, discovers additional salient \emph{pal-tuples} that were potentially affected, and adds them to the set of potentially affected \emph{pal-tuples} $\Gamma_{\delta}$.

To find \emph{pal-tuples} corresponding to
reduced functionality of the agent, we place the agent in an optimal planning mode and assume limited availability of observation traces $\mathbb{O}$ in the form of optimal unit-cost state-action trajectories $\langle s_0, a_1, s_1, a_2, \dots,
s_{n-1}, a_n, s_n \rangle$. 
We generate optimal plans using $\mi$ for all pairs of states in $\mathbb{O}$.
We hypothesize that, if for a pair of states, the plan generated using
$\mi$ is shorter than the plan observed in $\mathbb{O}$, then some functionality of the agent has reduced.

Our method performs comparative analysis of optimality of the observation traces against the optimal solutions generated using $\mi$ for same pairs of initial and final states. To begin with, we extract all the continuous state sub-sequences from $\mathbb{O}$ of the form $\langle s_0, s_1, \dots, s_n \rangle$ denoted by $\mathbb{O}_{\emph{drift}}$  as they are all optimal. We then generate a set of planning problems $\mathcal{P}$ using the initial and final states of trajectories in $\mathbb{O}_\emph{drift}$. Then, we provide the problems in $\mathcal{P}$ to $\mi$ to get a set of optimal trajectories $\mathbb{O}_\emph{init}$. 
We select all the pairs of optimal trajectories of the form $\langle o_{\emph{init}}, o_{\emph{drift}}\rangle$ for further analysis such that
the length of $o_{\emph{init}} \in \mathbb{O}_{\emph{init}}$ for a problem is shorter than the length of $o_{\emph{drift}} \in \mathbb{O}_{\emph{drift}}$ for the same problem. For all such pairs of optimal trajectories, a subset of actions in each $o_\emph{init} \in \mathbb{O}_\emph{init}$ were likely affected due to the model drift. We focus on identifying the first action in each $o_\emph{init} \in \mathbb{O}_\emph{init}$ that was definitely affected.

To identify the affected actions, we traverse each pair of optimal trajectories $\langle o_{\emph{init}}, o_{\emph{drift}}\rangle$ simultaneously starting from the initial states. We add all the \emph{pal-tuples} corresponding to the first differing action in $o_\emph{init}$ to $\Gamma_{\delta}$. We do this because there are only two possible explanations for why the action differs: (i) either the action in $o_\emph{init}$ was applicable in a state using $\mi$ but has become inapplicable in the same state in $\md$, or (ii) it can no longer achieve the same effects in $\md$ as $\mi$. We also discover the first actions that are applicable in the same states in both the trajectories but result in different states. The effect of such actions has certainly changed in $\md$. We add all the \emph{pal-tuples} corresponding to such actions to $\Gamma_{\delta}$.
In the next section, we describe our approach to infer the correct modes of \emph{pal-tuples} in $\Gamma_{\delta}$.

\subsection{Investigating Affected pal-tuples}
\label{sec:dmaa}

This section explains how the correct modes
of \emph{pal-tuples} in $\Gamma_\delta$ are inferred (line 2 onwards of Alg.1).
Alg.~\ref{alg:alg3} creates an abstract model in which all
the \emph{pal-tuples} that are predicted to have been affected are
set to $\unkwn$ (line 2). It then iterates over all
\emph{pal-tuples} with mode $\unkwn$ (line 4).

\subsubsection{Removing inconsistent models}
Our method generates candidate abstract models and then removes the
abstract models that are not consistent with the agent (lines 7-18
of Alg.~\ref{alg:alg3}).
For each \emph{pal-tuple} $\gamma \in \Gamma$, the algorithm 
computes a set of possible abstract models  $\mathcal{M}_\emph{abs}$
by assigning the three mode variants $+$, $-$, and $\abs$ to the 
current \emph{pal-tuple} $\gamma$ in model $M_\emph{abs}$ (line 6).
Only one model in $\mathcal{M}_\emph{abs}$ corresponds to the
agent's updated functionality. 

If the action $\gamma_a$ in the \emph{pal-tuple} $\gamma$ is present in the set of action triplets generated using
$\mathbb{O}$, then the pre-state of that action $s_{\emph{pre}}$
is used to create a state $s_I$ (lines 9-10). $s_I$ is created by removing the
literals corresponding to predicate $\gamma_p$ from $s_\emph{pre}$.
We then create a query $Q\! = \!\tuple{s_I, \tuple{\gamma_a}}$ (line~10), and pose
it to the agent $\ag$ (line 11).
The three models are then sieved based on the comparison of their
responses to the query $Q$ with that of $\ag$'s response $\theta$ to $Q$ (line 12). We use
the same mechanism as AIA for sieving the abstract models.

If the action corresponding to the current \emph{pal-tuple} $\gamma$ being
considered is not present in any of the observed action triplets,
then for every pair of abstract models in $\mathcal{M}_\emph{abs}$ (line 14),
we generate a query $Q$ using a planning problem (line 15).
We then pose the query $Q$ to the agent (line 16) and receive its response $\theta$.
We then sieve the abstract models by asking them the same query and 
discarding the models whose responses are not consistent with
that of the agent (line 17).
The planning problem that is used to generate the query and the
method that checks for consistency of abstract models' responses
with that of the agent are used from AIA.

Finally, all the models that are not consistent
with the agent's updated functionality are removed from the 
possible set of models $\m_\emph{abs}$. The remaining models are returned by the algorithm.
Empirically, we find that only one model is always returned by the algorithm.

\subsection{Correctness}
We now show that the learned drifted model representing the agent's updated functionality is consistent as defined in Def.~8 and Def.~9.
The proof of the theorem is available in the extended version of the paper~\cite{daaisy_extended}.

\begin{theorem}
Given a set of observation traces $\mathbb{O}$ generated by the 
  drifted agent $\ad$,
   a set of queries $Q$ posed to $\ad$ by Alg.~1, and the model $\mi$ representing the agent's functionality prior to the drift, each of the models $M = \langle P, A\rangle$ in $\mdset$ learned by Alg.~1 are \emph{consistent} with respect to all the observation traces $o \in \mathbb{O}$ and query-responses $\tuple{q,\theta}$ for all the queries $q \in Q$.
\end{theorem}
There exists a finite set of observations that if collected will allow Alg.~\ref{alg:alg3} to achieve 100\% correctness with any amount of drift: 
this set corresponds to observations that allow line 1 of Alg.~\ref{alg:alg3} to detect a change in the functionality. This includes an action triplet in 
an observation trace hinting at increased functionality, 
or a shorter plan using the previously known model hinting at reduced functionality. Thus, models learned by DAAISy are guaranteed to be completely correct irrespective of the amount of the drift if such a finite set of observations is available.
While using queries significantly reduces the number of observations required, asymptotic guarantees subsume those of passive model learners while ensuring convergence to the true model. 

\section{Empirical Evaluation}
\label{sec:experiments}

In this section, we evaluate our approach for assessing a black-box agent to learn its
model using information about its previous model and available observations. We implemented 
the algorithm for DAAISy in 
Python\!\!~\footnote{Code available at https://github.com/AAIR-lab/DAAISy} and tested it on six
planning benchmark domains from the International
Planning Competition 
(IPC)\!~\footnote{https://www.icaps-conference.org/competitions}. 
We used the IPC domains as the
unknown drifted models and generated six initial domains at random
for each domain in our experiments. 

To assess the performance of
our approach with increasing drift, we employed two methods for generating the
initial domains: 
(a) dropping the \emph{pal-tuples} already present, and (b) adding new \emph{pal-tuples}.
For each experiment, we used both types of domain generation. We
generated different initial models by randomly changing modes of random
\emph{pal-tuples} in the IPC domains. Thus, in all our experiments an IPC
domain plays the role of ground truth $M^*_\emph{drift}$ and a randomized model is
used as $\mi$. 

We use a very small set of observation traces $\mathbb{O}$ (single observation trace containing 10 action triplets) in all the experiments
for each domain.  
To generate this set, we gave the agent a random 
problem instance from the IPC corresponding to the domain used by the agent.
The agent then used Fast Downward~\cite{Helmert06thefast} 
with LM-Cut heuristic~\cite{Helmert2009LandmarksCP} to produce an optimal solution
for the given problem.
The generated observation trace is
provided to DAAISy as input in addition to a random $\mi$
as discussed in Alg.~\ref{alg:alg3}. 
The exact same observation trace is used in all experiments of the same domain, without the knowledge of the drifted model of the agent, and
irrespective of the amount of drift. 

We measure the final accuracy of the learned model $\md$ against the ground
truth model $M^*_\emph{drift}$ using the measure of model difference
$\Delta(\md, M^*_\emph{drift})$. 
We also measure the number of queries required to learn a model with significantly high accuracy. We compare the efficiency of DAAISy (our approach) with the Agent Interrogation Algorithm (AIA) \cite{verma2021asking} as it is the most closely related querying-based system. 

All of our experiments were executed on 5.0 GHz Intel i9 CPUs with 64 GB RAM running Ubuntu 18.04. We now discuss our results in detail below.\\

\subsection{Results}
\begin{figure}[t]
\centering
\centering
\includegraphics[width=\linewidth]{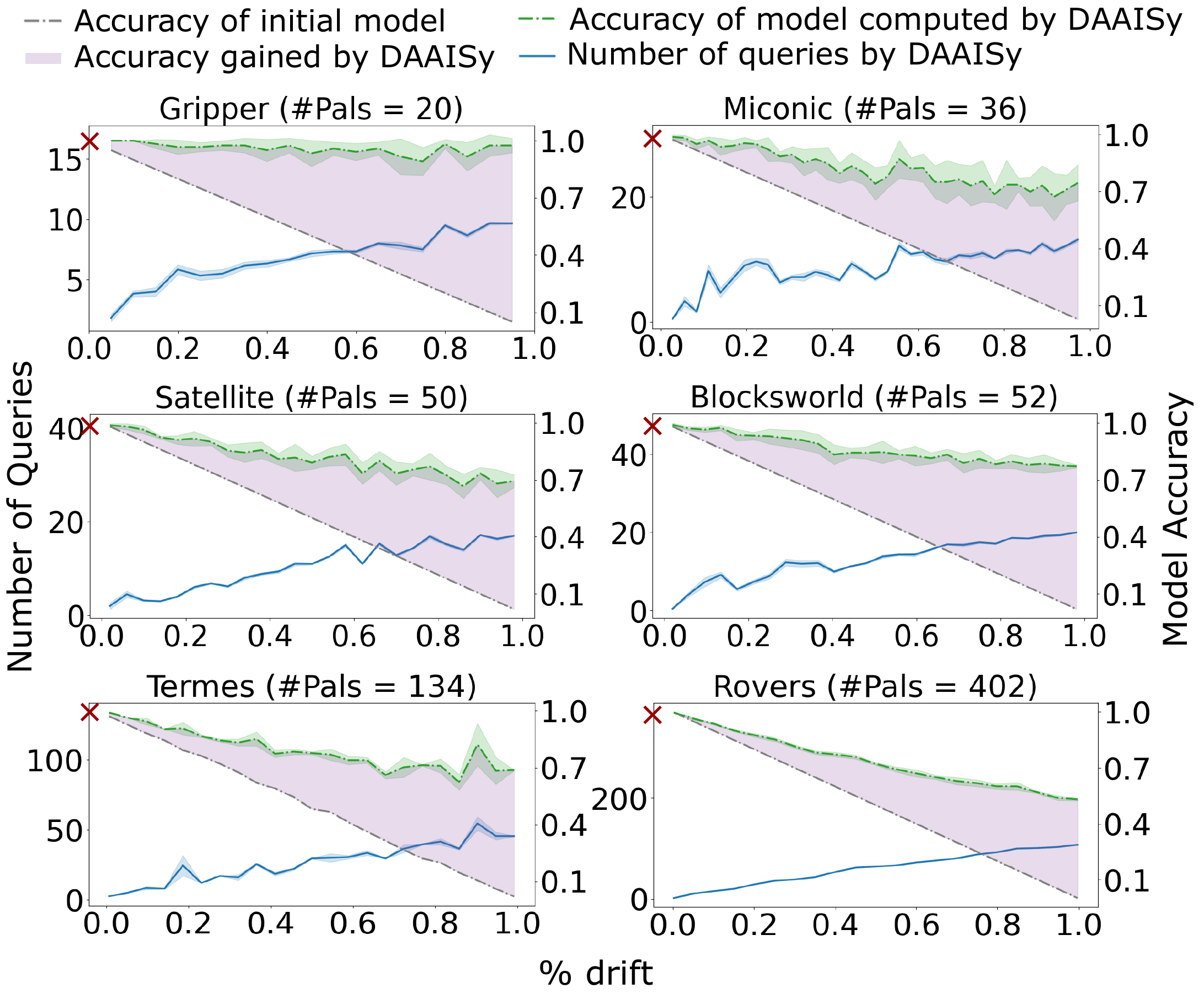}
\caption{\label{fig:result} The number of queries used by DAAISy (our approach) and AIA (marked {\textcolor{aia-x-color}{$\mytimes$}} on y-axis), as well as accuracy of model computed by DAAISy with increasing amount of drift. Amount of drift equals the ratio of drifted \emph{pal-tuples} and the total number of \emph{pal-tuples} in the domains (nPals). The number of action triplets in the observation trace used for each domain is 10. }
\end{figure}

We evaluated the performance of DAAISy along 2 directions; the number of queries it takes to learn the updated model $\md$ with increasing amount of drift, and the correctness of the model $\md$ it learns compared to $M^*_\emph{drift}$.

\subsubsection{Efficiency in number of queries} 
As seen in Fig.~\ref{fig:result},
the computational cost of assessing each agent, measured in terms of the number of queries used by DAAISy, increases as the amount of drift in the model $M^*_\emph{drift}$ increases. 
This is expected as the amount of drift is directly proportional to the number of \emph{pal-tuples} affected in the domain. This increases the number of \emph{pal-tuples} that DAAISy identifies as affected as well as the number of queries as a result.
As demonstrated in the plots, the standard deviation for number of queries remains low even when we increase the amount of drift, showing the stability of DAAISy.
	
\subsubsection{Comparison with AIA} 
Tab.~\ref{tab:t3} 
shows the average number of queries that AIA took to achieve the same level of accuracy as our approach for 50\% drifted models, and DAAISy requires significantly fewer queries to reach the same levels of accuracy compared to AIA. Fig.~\ref{fig:result} also demonstrates that DAAISy always takes fewer queries as compared to AIA to reach reasonably high levels of accuracy.

This is because AIA does not use information about the previously known model of the
agent and thus ends up querying for all possible \emph{pal-tuples}. DAAISy, on the other hand, predicts the set of \emph{pal-tuples} that might have changed based on the observations collected from the agent and thus requires significantly fewer queries.

\subsubsection{Correctness of learned model}
DAAISy computes models with at least 50\% accuracy in all six domains even when they have completely drifted from their initial model, i.e., $\Delta(\md, M^*_\emph{drift}) = $ \emph{nPals}. It attains nearly accurate models for Gripper and Blocksworld for upto 40\% drift. 
Even in scenarios where the agent’s model drift is more than 50\%, DAAISy achieves at least
70\% accuracy in five domains.
Note that DAAISy is guaranteed to find the correct mode for an identified affected \emph{pal-tuple}.
The reason for less than $100\%$ accuracy when using DAAISy is that it 
does not predict a \emph{pal-tuple} to be affected unless it encounters an observation trace conflicting with $\mi$. Thus, the learned model $\md$, even though consistent with all the observation traces, may end up being inaccurate when compared to $M^*_\emph{drift}$.

\begin{table}[t]
	\begin{tabular}{l c S[table-format=3.2] S[table-format=3.2]} 
	\toprule
	Domain & \#Pals & AIA & {DAAISy}  \\
	\midrule
	Gripper & 20 & 15.0 & 6.5 \\ 
	Miconic & 36 & 32.0 & 7.7  \\ 
	Satellite & 50 & 34.0 & 9.0 \\
	Blocksworld & 52 & 40.0 & 11.4  \\
	Termes & 134 & 115.0 & 27.0  \\
	Rovers & 402 & 316.0 & 61.0 \\
	\bottomrule
	\end{tabular}
	\centering
	\caption{The average number of queries taken by AIA to achieve the same level of accuracy as DAAISy (our approach) for 50\% drifted models.}
	\label{tab:t3}
\end{table}

\subsubsection{Discussion} AIA always ends up learning completely accurate models, but as noted above, this is because AIA queries exhaustively for all the \emph{pal-tuples} in the model. There is a clear trade-off between the number of queries that DAAISy takes to learn the model as compared to AIA and the correctness of the learned model. As evident from the results, if the model has not drifted much, DAAISy can serve as a better approach to
efficiently
learn the updated functionality of the agent with less overhead as compared to AIA. Deciding the amount of drift after which it would make sense to switch to querying the model from scratch is a useful analysis not addressed in this paper.

\section{Related Work}
\label{sec:related_work}

\subsubsection{White-box model drift} \citet{Bryce2016} address the problem of learning the updated mental model of a user using particle filtering given prior knowledge about the user's mental model. However, they assume that the entity being modeled can
tell the learning system about flaws in the learned model if needed. 
 \citet{eiter_2005_updating,eiter_2010_updating} propose a framework for updating action laws depicted in the form of graphs representing the state space. They assume that changes can only happen in effects, and that knowledge about the state space and what effects might change is available beforehand. 
Our work does not make such assumptions to learn the correct model of the agent's functionalities.

\subsubsection{Action model learning} The problem of learning agent models from observations of its
behavior is an active area of research~\cite{gil_94_learning,Yang2007,Cresswell09,Zhuo13action,arora_2018_review,aineto2019learning}. Recent work addresses active querying to learn the action model of an agent~\cite{rodrigues_2011_active,verma2021asking}.
However, these methods do not address the problem of reducing the computational cost of differential model assessment, which is crucial in non-stationary settings.

Online action model learning approaches learn the model of an agent while incorporating new observations of the agent behavior~\cite{certicky_2014_real,lamanna_2021_online2,lamanna_2021_online}. Unlike our approach, they do not handle cases where (i) the new observations are not consistent with the older ones due to changes in the agent's behavior; and/or (ii) there is reduction in functionality of the agent.
\citet{lindsay_2021_reuniting} solve the problem of learning all static predicates in a domain. They start with a correct partial model that captures the dynamic part of the model accurately and generate negative examples by assuming access to all possible positive examples. Our method is different in that it does not make such assumptions and 
leverages a small set of available observations to infer about increased and reduced functionality of an agent's model.

\subsubsection{Model reconciliation} Model reconciliation literature~\cite{chakraborti_2017_plan,sreedharan_2019_model, sreedharan_2021_foundations} 
deals with inferring the differences between the user and the agent models
and removing them using explanations. These methods consider white-box known models
whereas our approach works with black-box models of the agent. 

\section{Conclusions and Future Work}
\label{sec:conclusion}

We presented a novel method for \emph{differential assessment} of black-box AI systems to learn models of true functionality of agents that have drifted from their previously known functionality. Our approach provides guarantees of correctness w.r.t. observations. Our evaluation demonstrates that our system, DAAISy, efficiently learns a highly accurate model of agent's functionality issuing a significantly lower number of queries as opposed to relearning from scratch.
In the future, we plan to extend the framework to more general classes, stochastic settings, and models.
Analyzing and predicting switching points from selective querying in DAAISy to relearning from scratch without compromising the correctness of the learned models is also a promising direction for future work.

\section*{Acknowledgements} \label{sec:ack}
We thank anonymous reviewers for their helpful feedback on the paper.
This work was supported in part by the NSF under grants IIS 1942856,
IIS 1909370, and the ONR grant N00014-21-1-2045.

\bibliography{drift}

\begin{thebibliography}{31}
\providecommand{\natexlab}[1]{#1}

\bibitem[{Aineto, Celorrio, and Onaindia(2019)}]{aineto2019learning}
Aineto, D.; Celorrio, S.~J.; and Onaindia, E. 2019.
\newblock Learning Action Models With Minimal Observability.
\newblock \emph{Artificial Intelligence}, 275: 104--137.

\bibitem[{Arora et~al.(2018)Arora, Fiorino, Pellier, Métivier, and
  Pesty}]{arora_2018_review}
Arora, A.; Fiorino, H.; Pellier, D.; Métivier, M.; and Pesty, S. 2018.
\newblock {A Review of Learning Planning Action Models}.
\newblock \emph{The Knowledge Engineering Review}, 33: E20.

\bibitem[{B\"{a}ckstr\"{o}m and Jonsson(2013)}]{Backstrom13}
B\"{a}ckstr\"{o}m, C.; and Jonsson, P. 2013.
\newblock Bridging the Gap Between Refinement and Heuristics in Abstraction.
\newblock In \emph{Proc. IJCAI}.

\bibitem[{Bryce, Benton, and Boldt(2016)}]{Bryce2016}
Bryce, D.; Benton, J.; and Boldt, M.~W. 2016.
\newblock Maintaining Evolving Domain Models.
\newblock In \emph{Proc. IJCAI}.

\bibitem[{{\v{C}}ertick\'{y}(2014)}]{certicky_2014_real}
{\v{C}}ertick\'{y}, M. 2014.
\newblock {Real-Time Action Model Learning with Online Algorithm 3SG}.
\newblock \emph{Applied Artificial Intelligence}, 28(7): 690–711.

\bibitem[{Chakraborti et~al.(2017)Chakraborti, Sreedharan, Zhang, and
  Kambhampati}]{chakraborti_2017_plan}
Chakraborti, T.; Sreedharan, S.; Zhang, Y.; and Kambhampati, S. 2017.
\newblock {Plan Explanations as Model Reconciliation: Moving Beyond Explanation
  as Soliloquy}.
\newblock In \emph{Proc. IJCAI}.

\bibitem[{Cresswell, McCluskey, and West(2009)}]{Cresswell09}
Cresswell, S.; McCluskey, T.; and West, M. 2009.
\newblock Acquisition of Object-Centred Domain Models from Planning Examples.
\newblock In \emph{Proc. ICAPS}.

\bibitem[{Eiter et~al.(2005)Eiter, Erdem, Fink, and
  Senko}]{eiter_2005_updating}
Eiter, T.; Erdem, E.; Fink, M.; and Senko, J. 2005.
\newblock {Updating Action Domain Descriptions}.
\newblock In \emph{Proc. IJCAI}.

\bibitem[{Eiter et~al.(2010)Eiter, Erdem, Fink, and
  Senko}]{eiter_2010_updating}
Eiter, T.; Erdem, E.; Fink, M.; and Senko, J. 2010.
\newblock {Updating Action Domain Descriptions}.
\newblock \emph{Artificial Intelligence}, 174(15): 1172–1221.

\bibitem[{Fikes and Nilsson(1971)}]{Fikes1971}
Fikes, R.~E.; and Nilsson, N.~J. 1971.
\newblock {STRIPS}: A New Approach to the Application of Theorem Proving to
  Problem Solving.
\newblock \emph{Artificial Intelligence}, 2(3-4): 189--208.

\bibitem[{Fox and Long(2003)}]{fox03_pddl}
Fox, M.; and Long, D. 2003.
\newblock {PDDL}2.1: An Extension to {PDDL} for Expressing Temporal Planning
  Domains.
\newblock \emph{Journal of Artificial Intelligence Research}, 20(1): 61--124.

\bibitem[{Gil(1994)}]{gil_94_learning}
Gil, Y. 1994.
\newblock Learning by Experimentation: Incremental Refinement of Incomplete
  Planning Domains.
\newblock In \emph{Proc. ICML}.

\bibitem[{Giunchiglia and Walsh(1992)}]{Giunchiglia92}
Giunchiglia, F.; and Walsh, T. 1992.
\newblock A Theory of Abstraction.
\newblock \emph{Artificial Intelligence}, 57(2-3): 323--389.

\bibitem[{Helmert(2006)}]{Helmert06thefast}
Helmert, M. 2006.
\newblock The Fast Downward Planning System.
\newblock \emph{Journal of Artificial Intelligence Research}, 26: 191--246.

\bibitem[{Helmert and Domshlak(2009)}]{Helmert2009LandmarksCP}
Helmert, M.; and Domshlak, C. 2009.
\newblock Landmarks, Critical Paths and Abstractions: What's the Difference
  Anyway?
\newblock In \emph{Proc. ICAPS}.

\bibitem[{Helmert, Haslum, and Hoffmann(2007)}]{Helmert07}
Helmert, M.; Haslum, P.; and Hoffmann, J. 2007.
\newblock Flexible Abstraction Heuristics for Optimal Sequential Planning.
\newblock In \emph{Proc. ICAPS}.

\bibitem[{Lamanna et~al.(2021{\natexlab{a}})Lamanna, Gerevini, Saetti,
  Serafini, and Traverso}]{lamanna_2021_online2}
Lamanna, L.; Gerevini, A.~E.; Saetti, A.; Serafini, L.; and Traverso, P.
  2021{\natexlab{a}}.
\newblock {On-line Learning of Planning Domains from Sensor Data in PAL:
  Scaling up to Large State Spaces}.
\newblock In \emph{Proc. AAAI}.

\bibitem[{Lamanna et~al.(2021{\natexlab{b}})Lamanna, Saetti, Serafini,
  Gerevini, and Traverso}]{lamanna_2021_online}
Lamanna, L.; Saetti, A.; Serafini, L.; Gerevini, A.; and Traverso, P.
  2021{\natexlab{b}}.
\newblock {Online Learning of Action Models for PDDL Planning}.
\newblock In \emph{Proc. IJCAI}.

\bibitem[{Lindsay(2021)}]{lindsay_2021_reuniting}
Lindsay, A. 2021.
\newblock {Reuniting the LOCM Family: An Alternative Method for Identifying
  Static Relationships}.
\newblock In \emph{ICAPS 2021 KEPS Workshop}.

\bibitem[{Long and Fox(2003)}]{long_2003_ipc}
Long, D.; and Fox, M. 2003.
\newblock {The 3rd International Planning Competition: Results and Analysis}.
\newblock \emph{Journal of Artificial Intelligence Research}, 20: 1--59.

\bibitem[{McDermott et~al.(1998)McDermott, Ghallab, Howe, Knoblock, Ram,
  Veloso, Weld, and Wilkins}]{McDermott_1998_PDDL}
McDermott, D.; Ghallab, M.; Howe, A.; Knoblock, C.; Ram, A.; Veloso, M.; Weld,
  D.~S.; and Wilkins, D. 1998.
\newblock PDDL -- The Planning Domain Definition Language.
\newblock Technical Report CVC TR-98-003/DCS TR-1165, Yale Center for
  Computational Vision and Control.

\bibitem[{Nayyar, Verma, and Srivastava(2022)}]{daaisy_extended}
Nayyar, R.~K.; Verma, P.; and Srivastava, S. 2022.
\newblock Differential Assessment of Black-Box AI Agents.
\newblock \emph{arXiv preprint arXiv: 2203.13236}.

\bibitem[{Rodrigues et~al.(2011)Rodrigues, G{\'e}rard, Rouveirol, and
  Soldano}]{rodrigues_2011_active}
Rodrigues, C.; G{\'e}rard, P.; Rouveirol, C.; and Soldano, H. 2011.
\newblock {Active Learning of Relational Action Models}.
\newblock In \emph{Proc. ILP}.

\bibitem[{Sacerdoti(1974)}]{Sacerdoti74}
Sacerdoti, E.~D. 1974.
\newblock Planning in a Hierarchy of Abstraction Spaces.
\newblock \emph{Artificial Intelligence}, 5(2): 115--135.

\bibitem[{Sreedharan, Chakraborti, and
  Kambhampati(2021)}]{sreedharan_2021_foundations}
Sreedharan, S.; Chakraborti, T.; and Kambhampati, S. 2021.
\newblock {Foundations of Explanations as Model Reconciliation}.
\newblock \emph{Artificial Intelligence}, 103558.

\bibitem[{Sreedharan et~al.(2019)Sreedharan, Hernandez, Mishra, and
  Kambhampati}]{sreedharan_2019_model}
Sreedharan, S.; Hernandez, A.~O.; Mishra, A.~P.; and Kambhampati, S. 2019.
\newblock {Model-Free Model Reconciliation}.
\newblock In \emph{Proc. IJCAI}.

\bibitem[{Srivastava, Russell, and Pinto(2016)}]{Srivastava16}
Srivastava, S.; Russell, S.; and Pinto, A. 2016.
\newblock Metaphysics of Planning Domain Descriptions.
\newblock In \emph{Proc. AAAI}.

\bibitem[{Stern and Juba(2017)}]{Stern2017}
Stern, R.; and Juba, B. 2017.
\newblock Efficient, Safe, and Probably Approximately Complete Learning of
  Action Models.
\newblock In \emph{Proc. IJCAI}.

\bibitem[{Verma, Marpally, and Srivastava(2021)}]{verma2021asking}
Verma, P.; Marpally, S.~R.; and Srivastava, S. 2021.
\newblock {Asking the Right Questions: Learning Interpretable Action Models
  Through Query Answering}.
\newblock In \emph{Proc. AAAI}.

\bibitem[{Yang, Wu, and Jiang(2007)}]{Yang2007}
Yang, Q.; Wu, K.; and Jiang, Y. 2007.
\newblock Learning Action Models from Plan Examples Using Weighted MAX-SAT.
\newblock \emph{Artificial Intelligence}, 171(2-3): 107--143.

\bibitem[{Zhuo and Kambhampati(2013)}]{Zhuo13action}
Zhuo, H.~H.; and Kambhampati, S. 2013.
\newblock Action-Model Acquisition from Noisy Plan Traces.
\newblock In \emph{Proc. IJCAI}.

\end{thebibliography}

\cleardoublepage
\appendix
\setcounter{theorem}{0}

\section{Proofs of Theoretical Results}
To prove Thm.~1 in the paper we will first need to prove that Tab.~2 is constructed correctly. We do this by using the following result:

\begin{lemma}
Given an action triplet $\langle s, a, s'\rangle \in \mathbb{O}$ and a predicate $p \in P$, Tab.~2 correctly represents the set of values for the pair of modes $\langle m_{\emph{pre}},m_{\emph{eff}} \rangle$, where $m_{\emph{pre}}$ and $m_{\emph{eff}}$ are the modes of predicate $p$ in the precondition and effect of action $a$ respectively, that are \emph{consistent} with the action triplet.
\end{lemma}

\begin{proof}
Given an action triplet $\langle s, a, s^\prime \rangle$, if a predicate $p \in P$ is true (or false) in $s$ (or $s^\prime$), then it cannot be false (or true) in the precondition (or effect) of $a$. Hence, if $p$ is true in both $s$ and $s^\prime$, its value for $\langle m_{\emph{pre}},m_{\emph{eff}} \rangle$ can only be $\langle +,\abs \rangle$, $\tuple{\abs,+}$, or $\tuple{\abs,\abs}$. If $p$ is true in $s$ but false in $s^\prime$, its value for $\langle m_{\emph{pre}},m_{\emph{eff}} \rangle$ can only be $\langle +,- \rangle$ or $\tuple{\abs,-}$. If $p$ is false in $s$ but true in $s^\prime$, its value for $\langle m_{\emph{pre}},m_{\emph{eff}} \rangle$ can only be $\langle -,+ \rangle$ or $\tuple{\abs,+}$. Finally, if $p$ is false in both $s$ and $s^\prime$, its value for $\langle m_{\emph{pre}},m_{\emph{eff}} \rangle$ can only be $\langle -,\abs \rangle$, $\tuple{\abs,-}$, or $\tuple{\abs,\abs}$. Thus, for an observed action triplet $\langle s, a, s' \rangle$, Tab.~2 shows all the possible values for a $p \in P$ in the precondition and effect of $a$ that do not conflict with the presence (or absence) of $p$ in $s$ and $s'$ respectively.
\end{proof}

We now prove the two smaller results that combine to form Theorem~1.

\begin{lemma}
Given a set of observation traces $\mathbb{O}$ generated by the 
drifted agent $\ad$, each of the models $M = \langle P, A\rangle$ in $\mdset$ learned by Alg.~1 are \emph{consistent} with respect to all the observation traces $o \in \mathbb{O}$.
\end{lemma}

\begin{proof}
Given that the action triplets in the set of observations $\mathbb{O}$ are generated using the same functionality of the deterministic agent after the drift $\ag_{\emph{drift}}$ (i.e., all the observations correspond to the same drifted model), any two different action triplets in $\mathbb{O}$ containing groundings
of the same action $a_i$ must have pre- and post-states that do not contradict each other.
Now, for any action triplet $\tuple{s_{i-1}, a_i, s_i}$ that is part of an observation trace $o \in \mathbb{O}$, 
when we consider the correct values for $\tuple{m_{\emph{pre}},m_{\emph{eff}}}$ for a \emph{pa-tuple} $\tuple{p,a_i}$
such that $p \in P$,
we only consider the values for $\tuple{m_{\emph{pre}},m_{\emph{eff}}}$ that are shown in Tab.~2. For multiple actions triplets, possible values for $\tuple{m_{\emph{pre}},m_{\emph{eff}}}$ can be found by taking an intersection of the sets of values for $\tuple{m_{\emph{pre}},m_{\emph{eff}}}$ for each action triplet found using Tab.~2.
Using Lemma 1, this ensures that the learned model $M$ is consistent with all the action triplets in
an observation trace $o \in \mathbb{O}$. Since an observation trace is a sequence of action triplets, the learned model $M \in \mdset$ is consistent with all the observation traces in the set of observation traces $\mathbb{O}$.
\end{proof}

\begin{lemma}
Given a set of queries $Q$ posed to $\ad$ by Alg.~1, and the model $\mi$ representing the agent's functionality prior to the drift, each of the models $M = \langle P, A\rangle$ in $\mdset$ learned by Alg.~1 are \emph{consistent} with respect to all query-responses $\tuple{q,\theta}$ for all the queries $q \in Q$.
\end{lemma}

\begin{proof}
The agent responds to the query $q = \langle s_I, \pi= \langle a_1,\dots a_n \rangle \rangle$ using the drifted model with a response $\theta =\langle n_F, s_F\rangle$. This response can only be generated if there exists an observation trace $\langle s_I, a_1, s_1, \ldots, a_{n_F}, s_{n_F} \rangle$ of length $n_F$ that can take the agent starting from state $s_I$ to the state $s_{n_F}$. Now, pruning models based on responses of the agent follows
the criteria shown in Tab.~\ref{tab:t2}. Hence, the only modes we
consider for $p$ in the precondition and effect of $a_{n_F}$ are the ones that do not
conflict with the presence (or absence) of $p$ in $s_{n_{F} -1}$ and $s_{n_F}$ respectively. The modes for any predicate in other actions are not fixed using responses to queries. Hence, the learned model $M \in \mdset$ is consistent with all the query-responses $\tuple{q,\theta}$ for all the queries $q \in Q$.
\end{proof}

\setcounter{theorem}{0}
\begin{theorem}
Given a set of observation traces $\mathbb{O}$ generated by the 
  drifted agent $\ad$,
  a set of queries $Q$ posed to $\ad$ by Alg.~1, and the model $\mi$ representing the agent's functionality prior to the drift, each of the models $M = \langle P, A\rangle$ in $\mdset$ learned by Alg.~1 are \emph{consistent} with respect to all the observation traces $o \in \mathbb{O}$ and query-responses $\tuple{q,\theta}$ for all the queries $q \in Q$.
\end{theorem}

\begin{proof}
This theorem is conjunction of Lemma 1 and Lemma 2. Since both of the lemmas are proven to be true, this theorem is also true.
\end{proof}

\end{document}